\documentclass[letterpaper, 10 pt, conference]{ieeeconf}  %
\IEEEoverridecommandlockouts                              %
\title{Non-Monotone Energy-Aware Information Gathering\\for Heterogeneous Robot Teams%
}
\author{Xiaoyi Cai$^{1,*}$,
        Brent Schlotfeldt$^{2,*}$,
        Kasra Khosoussi$^1$,\\
        Nikolay Atanasov$^3$,
        George J. Pappas$^2$,
        and Jonathan P. How$^1$%
\thanks{$^*$indicates equal contribution.}
\thanks{$^1$X.\ Cai, K.\ Khosoussi, and J.\ P.\ How are with the Department of Aeronautics and Astronautics, Massachusetts Institute of Technology, Cambridge, MA 02139, USA.
	    {\tt\{xyc, kasra, jhow\}@mit.edu}.}
\thanks{$^2$B.\ Schlotfeldt, and G.\ J.\ Pappas are with the GRASP Laboratory, University of Pennsylvania, Philadelphia, PA 19104, USA.
	    {\tt\{brentsc, pappasg\}@seas.upenn.edu}.}
\thanks{$^3$N.\ Atanasov is with the Electrical and Computer Engineering Department, University of California, San Diego, La Jolla, CA 92093, USA.
	    {\tt natanasov@ucsd.edu}.}}

\usepackage{amsmath,amsthm,amssymb, amsfonts}
\usepackage{mathtools}
\usepackage{graphicx, xcolor}
\usepackage{pdfpages}
\usepackage{bm}
\usepackage{url}
\usepackage{enumerate}
\usepackage{float}
\usepackage[sort, compress]{cite}
\usepackage{cases,balance}
\usepackage[pdftex, pdfstartview={FitV}, pdfpagelayout={TwoColumnLeft},bookmarksopen=true,plainpages = false, colorlinks=true, linkcolor=black, citecolor = black, urlcolor = black,filecolor=black , pagebackref=false,hypertexnames=false, plainpages=false, pdfpagelabels ]{hyperref}

\usepackage{array}
\usepackage{booktabs}
\usepackage{multirow} %

\usepackage{comment} %
\excludecomment{comments} %

\usepackage{algorithm}
\usepackage[noend]{algpseudocode} %
\definecolor{commentclr}{RGB}{110, 149, 204}
\newcommand{\algcomment}[1]{{\color{commentclr}// #1}}
\makeatletter
\newcommand\fs@spaceruled{\def\@fs@cfont{\bfseries}\let\@fs@capt\floatc@ruled
  \def\@fs@pre{\vspace{0.6\baselineskip}\hrule height.8pt depth0pt \kern2pt}%
  \def\@fs@post{\kern2pt\hrule\relax}%
  \def\@fs@mid{\kern2pt\hrule\kern2pt}%
  \let\@fs@iftopcapt\iftrue}
\makeatother

\usepackage{tikz}
\newcommand*\circled[1]{\tikz[baseline=(char.base)]{
            \node[shape=circle,draw,inner sep= 0.3pt] (char) {#1};}}

\usepackage{lipsum}

\newcommand\copyrighttext{%
  \footnotesize © 2021 IEEE.  Personal use of this material is permitted. Permission from IEEE must be obtained for all other uses, in any current or future media, including reprinting/republishing this material for advertising or promotional purposes, creating new collective works, for resale or redistribution to servers or lists, or reuse of any copyrighted component of this work in other works.}
\newcommand\copyrightnotice{%
\begin{tikzpicture}[remember picture,overlay]
\node[anchor=south,yshift=10pt] at (current page.south) {\fbox{\parbox{\dimexpr\textwidth-\fboxsep-\fboxrule\relax}{\copyrighttext}}};
\end{tikzpicture}%
}

\newcommand{\R}{\mathbb{R}}

\DeclareMathOperator*{\argmax}{arg\,max}

\newcommand{\eg}{e.g.}
\newcommand{\ie}{i.e.}
\newcommand{\inv}{^{-1}}

\newcommand{\tr}{^\top}
\newcommand*{\defeq}{:=}
\newcommand{\logdet}{\log\det}

\theoremstyle{plain}%

\newtheorem{prop}{Proposition}

\theoremstyle{remark}
\newtheorem{remark}{Remark}

\theoremstyle{definition}

\newtheorem{problem}{Problem}
\newtheorem{definition}{Definition}

\newcommand{\cctrl}{c^{\mathrm{ctrl}}}

\newcommand{\cstate}{c^{\mathrm{state}}}

\newcommand{\cd}{\texttt{CD}}
\newcommand{\dls}{\texttt{DLS}}
\newcommand{\cls}{\texttt{CLS}}
\newcommand{\lexchange}{\texttt{FindProposal}}
\newcommand{\nop}{\small{\texttt{NOP}}}

\newcommand{\lsoffset}{O} %
\newcommand{\cmax}{c^{\mathrm{max}}}
\newcommand{\sopt}{S^*}
\newcommand{\sls}{S^{\mathrm{ls}}}

\newcommand{\mi}{\mathbb{I}}

\graphicspath{{Figs/}}

\graphicspath{{Algorithms/}}

\graphicspath{{Proofs/}}

\begin{document}

\maketitle
\begin{abstract}
This paper considers the problem of planning trajectories for a team of sensor-equipped robots to reduce uncertainty about a dynamical process. Optimizing the trade-off between information gain and energy cost (e.g., control effort, distance travelled) is desirable but leads to a non-monotone objective function in the set of robot trajectories. Therefore, common multi-robot planning algorithms based on techniques such as coordinate descent lose their performance guarantees. Methods based on local search provide performance guarantees for optimizing a non-monotone submodular function, but require access to all robots' trajectories, making it not suitable for distributed execution. This work proposes a distributed planning approach based on local search and shows how lazy/greedy methods can be adopted to reduce the computation and communication of the approach. We demonstrate the efficacy of the proposed method by coordinating robot teams composed of both ground and aerial vehicles with different sensing/control profiles and evaluate the algorithm's performance in two target tracking scenarios. Compared to the naive distributed execution of local search, our approach saves up to 60\% communication and 80--92\% computation on average when coordinating up to 10 robots, while outperforming the coordinate descent based algorithm in achieving a desirable trade-off between sensing and energy cost.
\end{abstract}
{\small
\section*{Supplementary Material}
\begin{center}
\url{https://www.youtube.com/watch?v=xWgFi6fwex0}
\end{center}
}

\copyrightnotice
\vspace*{-0.25in}
\floatstyle{spaceruled}%
\restylefloat{algorithm}%

\section{Introduction}\label{sec:introduction}
Developments in sensing and mobility have enabled effective utilization of robot systems in autonomous mapping\cite{carlone2014active, atanasov2015decentralized, lopez2017multi, corah2019distributed}, search and rescue\cite{kumar2004robot, tian2018search, queralta2020collaborative}, 
and environmental monitoring~\cite{shkurti2012multi, lan2016rapidly, notomista2020persistification, popovic2020informative}. 
These tasks require spatiotemporal information collection which can be achieved more efficiently and accurately by larger robot teams, rather than relying on individual robots. Robot teams may take advantage of heterogeneous capabilities, require less storage and computation per robot, and may achieve better environment coverage in shorter time\cite{grocholsky2006cooperative,sukhatme2007design,tokekar2016sensor,rizk2019cooperative}.
Task-level performance is usually quantified by a measure of information gain, where typically the marginal improvements diminish given additional measurements (\textit{submodularity}), and adding new measurements does not reduce the objective (\textit{monotonicity}). Although planning optimally for multi-robot sensing trajectories is generally intractable, these two properties allow for \emph{near-optimal} approximation algorithms that scale to large robot teams, while providing worst-case guarantees. Additionally, practical implementations often need to consider various measures for energy expenditure, such as control effort or distance travelled. 
A common approach is to impose fixed budgets, which preserves submodularity and monotonicity of the objective, so that existing algorithms may still be used \cite{singh2009efficient, singh2009nonmyopic , jorgensen2018team}.

\begin{figure} [t]
\centering
\includegraphics[width=0.95\columnwidth]{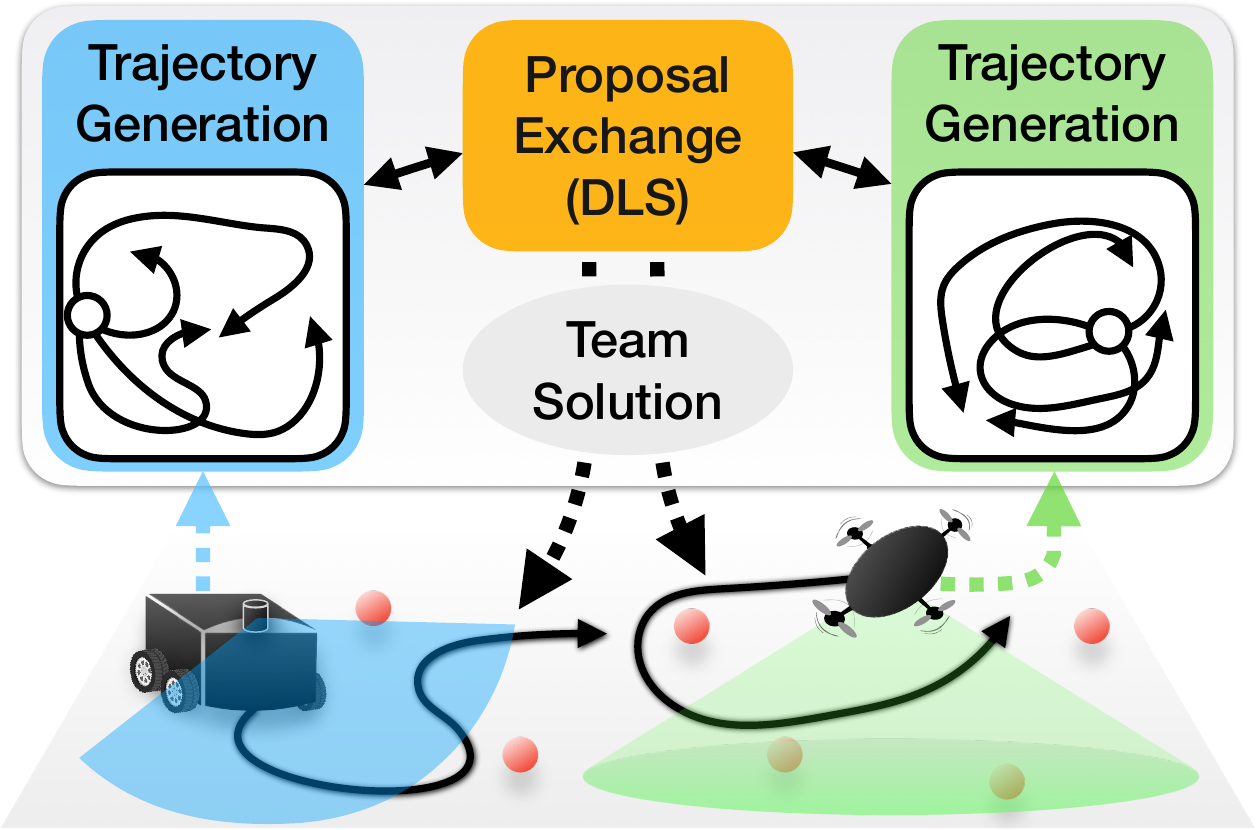}
\vspace*{-.1in}
\caption{Overview of the proposed distributed planning approach for non-monotone information gathering (see Sec.~\ref{sec:multi_robot_planning}). Robots generate individual candidate trajectories and jointly build a team plan via distributed local search (\dls{}), by repeatedly proposing changes to the collective trajectories.}
    \label{fig:front_cover}
\end{figure}

In this paper, we are motivated by scenarios where robots, with potentially different sensing and control capabilities, 
seek a desired trade-off between sensing and energy cost. Specifically, we formulate an \textit{energy-aware active information acquisition} problem, where the goal is to plan trajectories for a team of heterogeneous robots to maximize a weighted sum of information gain and energy cost. One key observation
is that adding the energy cost breaks the \textit{monotonicity} of the objective, violating an assumption held by existing approximation algorithms.  
Thus, we propose a new distributed planning algorithm based on local search~\cite{lee2009non} (see Fig.~\ref{fig:front_cover}) that has a worst-case guarantee for the non-monotone objective. We also show how to reduce the method's computation and communication to improve scalability.

\noindent\textit{\textbf{Related Work.}}
Our work belongs to the category of multi-robot informative path planning, where robots plan sensing trajectories to reduce uncertainty about a dynamic process (\eg,\cite{singh2009efficient,atanasov2015decentralized,schlotfeldt2018anytime,kantaros2019asymptotically,corah2019distributed, jorgensen2018team, dames2017detecting, viseras2020distributed, levine2010information, best2019dec}).
To alleviate the computational complexity, which is exponential in the number of robots, approximation methods have been developed to produce near-optimal solutions for a submodular and monotone objective (\eg, mutual information).
A common technique is coordinate descent, where robots plan successively while incorporating the plans of previous robots. %
Ref.~\cite{singh2009efficient} showed that coordinate descent extends the near-optimality of a single-robot planner to the multi-robot scenario.
This result was extend to dynamic targets by~\cite{atanasov2015active}, achieving at least $50\%$ of the optimal performance regardless of the planning order. 
Refs.~\cite{jorgensen2018team, dames2017detecting} decentralized the greedy method~\cite{fisher1978analysis} by adding the best single-robot trajectory to the team solution in every round.
Ref.~\cite{corah2019distributed} proposed distributed sequential greedy algorithm to alleviate the inefficiency in sequential planning.
 
Our problem can be seen as non-monotone submodular maximization subject to a partition matroid constraint (see Sec.~\ref{sec:problem_formulation}), for which approximation algorithms already exist.
The first such algorithm was developed by~\cite{lee2009non} based on local search, which can handle multiple matroid constraints. Extending~\cite{lee2009non}, ref.~\cite{gupta2010constrained} proposed a greedy-based approach that can handle multiple independence systems (more general than matroids), but has a worse approximation ratio given a single matroid. Other methods use multilinear relaxation such as~\cite{feldman2011unified, gharan2011submodular} for better approximation ratios, but require significant computation. 
Applying some of these ideas in robotics, ref.~\cite{segui2015decentralised} used the continuous greedy method by~\cite{feldman2011unified} for decentralized multi-robot task assignment. In the same domain, ref.~\cite{shin2019sample} combined sampling, greedy method, and lazy evaluation~\cite{minoux1978accelerated} to achieve fast computation.
We decided to build upon~\cite{lee2009non} for its simplicity and guarantees. We also attempt to incorporate well-known techniques like greedy method and lazy evaluation, but they are specialized in the context of local search, as detailed in Sec.~\ref{sec:dls}.

\noindent\textit{\textbf{Contributions.}} 
The main limitation of the prior works is the assumption of monotonicity of the objective function. Problems without monotonicity, such as the energy-aware problem we propose, cannot be solved by the above methods while retaining their near-optimality properties. In contrast, our proposed algorithm provides a theoretical performance guarantee even for non-monotone objectives. In this work:

\begin{itemize}
    \item We propose a distributed algorithm based on local search where robots collaboratively build a team plan by proposing modifications to the collective trajectories;
    \item We reduce its computation and communication requirements by prioritizing search orders of local search and warm starting with greedy solutions, respectively;
    \item We show that the proposed algorithm outperforms a state-of-the-art algorithm for multi-robot target tracking in coordinating a team of heterogeneous robots, while trading off sensing performance and energy expenditure.
\end{itemize}
\section{Preliminaries}\label{sec:prelim}
We review some useful definitions. Let $g:2^\mathcal{M}\to \R$ be a set function defined on the ground set $\mathcal{M}$ consisting of finite elements. Let $g(a|S)\defeq g(S \cup \{a\}) - g(S)$ be the discrete derivative, or the marginal gain, of $g$ at $S$ with respect to $a$.
\begin{definition}[Submodularity]\label{def:submodularity}
Function $g$ is submodular if for any $S_1\subseteq S_2\subseteq\mathcal{M}$ and $a\in\mathcal{M}\backslash S_2$, $g(a|S_1) \geq g(a|S_2)$.
\end{definition}

\begin{definition}[Monotonicity]\label{def:monotonicity}
Function $g$ is monotone if for any $S_1\subseteq S_2 \subseteq \mathcal{M}$, $g(S_1) \leq g(S_2)$.
\end{definition}
\section{Problem Formulation}\label{sec:problem_formulation}
Consider robots indexed by $i\in \mathcal{R} \defeq  \{1,\dots,n\}$, whose states are $x_{i, t}\in\mathcal{X}_i$ at time $t = 0,\ldots,T$, and dynamics are:
\begin{equation}
\label{eq:robot_dynamics}
x_{i, t+1} = f_i(x_{i, t}, u_{i, t}),
\end{equation}
where $u_{i, t}\in\mathcal{U}_i$ is the control input and $\mathcal{U}_i$ is a finite set. We denote a control sequence as $\sigma_i = u_{i,0}, \dots, u_{i,T-1} \in \mathcal{U}_i^T$.

The robots' goal is to track targets with state $y\in\R^{d_y}$ that have the following linear-Gaussian motion model:
\begin{equation}
    y_{t+1} = A_t y_t + w_t,\enspace\  w_t \sim \mathcal{N}(0,W_t),
    \label{eq:target_motion_model}
\end{equation}
where $A_t\in\R^{d_y \times d_y}$ and $w_t$ is a zero-mean Gaussian noise with covariance $W_t \succeq 0$.
The robots have sensors that measure the target state subject to an observation model:
\begin{equation}
    z_{i,t} = H_{i,t }(x_{i,t})y_t+v_{i,t}(x_{i,t}),\enspace\  v_{i,t}\sim\mathcal{N}(0,V_{i,t}(x_{i,t})),
    \label{eq:sensor_observation_model}
\end{equation}
where $z_{i,t}\in\R^{d_{z_i}}$ is the measurement taken by robot $i$ in state $x_{i,t}$, $H_{i,t}(x_{i,t})\in\R^{d_{z_i} \times d_y}$, and $v_{i,t}(x_{i,t})$ is a state-dependent Gaussian noise, whose values are independent at any pair of times and across sensors.
The observation model is linear in target states but can be nonlinear in robot states. If it depends nonlinearly on target states, we can linearize it around an estimate of target states to get a linear model.

We assume every robot $i$ has access to $N_i$ control trajectories $\mathcal{M}_i = \{\sigma_i^k\}_{k=1}^{N_i}$ to choose from. Denote the set of all control trajectories as $\mathcal{M} = \cup_{i=1}^{n} \mathcal{M}_i$ and its size as ${N= |\mathcal{M}|}$.
Potential control trajectories can be generated by various single-robot information gathering algorithms such as~\cite{atanasov2014information, hollinger2014sampling, lan2013planning, levine2010information}.
The fact that every robot cannot execute more than one trajectory can be encoded as a partition matroid $(\mathcal{M}, \mathcal{I})$, where $\mathcal{M}$ is the ground set, and $\mathcal{I}= \{ S \subseteq \mathcal{M} \mid |S \cap \mathcal{M}_i| \leq 1\ \forall i\in\mathcal{R}\}$ consists of all admissible subsets of trajectories.
Given $S\in\mathcal{I}$, we denote the joint state of robots that have been assigned trajectories as $x_{S,t}$ at time $t$, and their indices as
$\mathcal{R}_S\defeq \{i \mid |\mathcal{M}_i \cap S| = 1\ \forall\ i\in\mathcal{R}\}$.
Also, denote the measurements up to time $t\leq T$ collected by robots $i\in\mathcal{R}_S$ who follow the trajectories in $S$ by $z_{S, 1:t}$.

Due to the linear-Gaussian assumptions in \eqref{eq:target_motion_model} and  \eqref{eq:sensor_observation_model}, the optimal estimator for the target states is a Kalman filter. The target estimate covariance $\Sigma_{S,t}$ at time $t$ resulting from robots $\mathcal{R}_S$ following trajectories in $S$ obeys:
\begin{equation}
    \Sigma_{S, t+1} = \rho_{S,t+1}^{\mathrm e} ( \rho_{t}^{\mathrm p} (\Sigma_{S,t}), x_{S, t+1}),
\end{equation}
where $\rho_{t}^{\mathrm p}(\cdot)$ and $\rho_{S,t}^{\mathrm e}(\cdot, \cdot)$ are the Kalman filter prediction and measurement updates, respectively:
\begin{equation*}
    \begin{aligned}
        \textbf{Predict:}&&\rho_{t}^{\mathrm p} (\Sigma) & \defeq A_t\Sigma A_t\tr +W_t, \\
        \textbf{Update:} && \rho_{S,t}^{\mathrm e} (\Sigma, x_{S, t}) & \defeq \left(\Sigma\inv + \sum_{i\in\mathcal{R}_S} M_{i,t} (x_{i,t}) \right)\inv,\\
        && M_{i,t} (x_{i,t})& \defeq H_{i,t}(x_{i,t}) V_{i,t} (x_{i,t})\inv H_{i,t}(x_{i,t})\tr.
    \end{aligned}
\end{equation*}

When choosing sensing trajectories, we want to capture the trade-off between sensing performance and energy expenditure, which is formalized below.
\begin{problem}[Energy-Aware Active Information Acquisition]
\label{prob:MA_fuel_aware}
Given initial states $x_{i,0}\in \mathcal{X}_i$ for every robot $i\in\mathcal{R}$, a prior distribution of target state $y_0$, and a finite planning horizon $T$, find a set of trajectories $S\in\mathcal{M}$ to optimize the following:
\begin{equation} \label{eq:sensing_and_energy_objective}
\max_{S \in \mathcal{I}}\  J(S) \defeq  \mi(y_{1:T}; z_{S,1:T}) -  C(S),
\end{equation}
where $\mi(y_{1:T}; z_{S,1:T})= \frac{1}{2}\sum_{t=1}^T \big[ \logdet\big( \rho_{t-1}^{\mathrm p} (\Sigma_{S,t-1})\big) - \logdet (\Sigma_{S,t}) \big] \geq 0 $ is the mutual information between target states and observations\footnote{Our problem differs from sensor placement problems that consider the mutual information between selected and not selected sensing locations.},
and $C:2^{\mathcal{M}}\to\R$ is defined as:
\begin{align}
C(S) & \defeq \sum_{\sigma_i \in S} r_i\,C_{i}(\sigma_i),\label{eq:fuel_C}
\end{align}
where $0\leq C_i(\cdot) \leq \cmax$ is a non-negative, bounded energy cost for robot $i$ to apply controls $\sigma_i$ weighted by $r_i \geq 0$.
\end{problem}

\begin{remark}
Robots are assumed to know others' motion models~\eqref{eq:robot_dynamics} and observation models~\eqref{eq:sensor_observation_model} before the mission, so that any robot can evaluate~\eqref{eq:sensing_and_energy_objective} given a set of trajectories.
\end{remark}

\begin{remark}
The optimization problem~\eqref{eq:sensing_and_energy_objective} is non-monotone, because adding extra trajectories may worsen the objective by incurring high energy cost $C(S)$. 
Thus, the constraint $S\in\mathcal{I}$ may not be tight, \ie, some robots may not get assigned trajectories. This property is useful when a large repository of heterogeneous robots is available but only a subset is necessary for the given tasks.
\end{remark}

\begin{remark}
The choice of \eqref{eq:sensing_and_energy_objective} is motivated by the energy-aware target tracking application. However, the proposed algorithm in Sec.~\ref{sec:multi_robot_planning} is applicable to any scenario where  $J(S)$ is a submodular set function that is not necessarily monotone, but can be made non-negative with a proper offset.
\end{remark}

Solving Problem~\ref{prob:MA_fuel_aware} is challenging because adding energy cost $C(S)$ breaks the monotonicity of the objective, a property required for approximation methods (e.g., coordinate descent \cite{atanasov2015decentralized} and greedy algorithm \cite{fisher1978analysis}) to maintain performance guarantees. This is because these methods only add elements to the solution set, which always improves a monotone objective, but can worsen the objective in our setting, and may yield arbitrarily poor performance. We now propose a new distributed algorithm based on local search~\cite{lee2009non}.

\section{Multi-Robot Planning}\label{sec:multi_robot_planning}
We first present how local search~\cite{lee2009non} can be used to solve Problem~\ref{prob:MA_fuel_aware} with near-optimal performance guarantee. Despite the guarantee, local search is not suitable for distributed robot teams, because it assumes access to all locally planned robot control trajectories which can be communication-expensive to gather. To address this problem, we propose a new distributed algorithm that exploits the structure of a partition matroid to allow robots to collaboratively build a team plan by repeatedly proposing changes to the collective trajectories. Moreover, we develop techniques to reduce its computation and communication to improve scalability.

In the following subsections, we denote $g:2^{\mathcal{M}}\to \R$ as the non-negative, submodular oracle function used by local search,
where the ground set $\mathcal{M}$ contains robot trajectories.

\subsection{Centralized Local Search (\cls{})}
We present the original local search~\cite{lee2009non} in our setting with a single partition matroid constraint. We refer to it as centralized local search (\texttt{CLS}, Alg.~\ref{alg:cls}) because it requires access to trajectories $\mathcal{M}$ from all robots. The algorithm proceeds in two rounds to find two candidate solutions $S_1, S_2\in\mathcal{I}$. In each round $k=1,2$, solution $S_k$ is initialized with a single-robot trajectory maximizing the objective (Line~\ref{cls:best_singleton}).
Repeatedly, $S_k$ is modified by executing one of the \textbf{Delete}, \textbf{Add} or \textbf{Swap} operations, if it improves the objective by at least $(1+\frac{\alpha}{N^4})$ of its original value (Lines~\ref{cls:start_local_ops}--\ref{cls:end_local_ops}), where $\alpha>0$ controls run-time and performance guarantee.
This procedure continues until $S_k$ is no longer updated, and the next round begins without considering $S_k$ in the ground set $\mathcal{M}$ (Line~\ref{cls:remove_sk}). 
Lastly, the better of $S_1$ and $S_2$ is returned.

\begin{algorithm}[t]
\caption{Centralized Local Search~\cite{lee2009non} (\cls)} 
\label{alg:cls}
\small %
\begin{algorithmic}[1]
\State \textbf{require} $\alpha>0$, ground set $\mathcal{M}$, admissible subsets $\mathcal{I}$, oracle $g$
\State $N {\footnotesize\leftarrow} | \mathcal{M} |$
\State $S_1, S_2 {\footnotesize\leftarrow} \emptyset$ 

\For{$k=1, 2$}
    \State $S_k {\footnotesize\leftarrow} \{ \argmax_{a \in \mathcal{M} } g(\{ a \}) \}$ \hfill\algcomment{Initialize with best traj.} \label{cls:best_singleton} 
    \State \textbf{while} resultant $S_k'$ from \circled{1}, \circled{2} or \circled{3} satisfies $S_k'\in\mathcal{I}$ and $g(S_k') \geq (1+ \frac{\alpha}{N^4}) g(S_k)$ \textbf{do} $S_k {\footnotesize\leftarrow} S_k'$ \algcomment{Repeat local operations}\label{cls:start_local_ops}
    \State \hspace*{0.3em}  \circled{1} \textbf{Delete:} $S_k' {\footnotesize\leftarrow} S_k \backslash \{d\}$, where $d\in S_k$
    \State \hspace*{0.3em} \circled{2} \textbf{Add:} $S_k' {\footnotesize\leftarrow} S_k \cup \{a\}$, where $a\in \mathcal{M}\backslash S_k$
    \State \hspace*{0.4em} \circled{3} \textbf{Swap:} $S_k' {\footnotesize\leftarrow} S_k \backslash \{d\} \cup \{a\}$, where $d\in S_k,\,a\in \mathcal{M}\backslash S_k$\label{cls:end_local_ops}

    \State $\mathcal{M} {\footnotesize\leftarrow} \mathcal{M} \backslash S_k$ \label{cls:remove_sk}
\EndFor

\State \textbf{return} $\argmax_{S \in \{S_1, S_2\}} g( S)$ \label{cls:return}
\end{algorithmic}
\end{algorithm}

One important requirement of \texttt{CLS} is that the objective function $g$ is non-negative. With the objective from Problem~\ref{prob:MA_fuel_aware}, this may not be true, so we add an offset $\lsoffset$. The next proposition provides a worst-case performance guarantee for applying Alg.~\ref{alg:cls} to Problem~\ref{prob:MA_fuel_aware} after properly offsetting the objective to be  non-negative.

\begin{prop}\label{thm:LS_guarantee}
Consider that we solve Problem~\ref{prob:MA_fuel_aware} whose objective is made non-negative by adding a constant offset: 
\begin{equation}\label{eq:LS_optimization}
    \max_{S\in\mathcal{I}} \ g(S) \defeq J(S) +\lsoffset,
\end{equation}
where $\lsoffset  \defeq 
\sum_{i=1}^{n}r_i \cmax$. 
Denote $\sopt$ and $\sls$ as the optimal solution and solution obtained by \cls{} (Alg.~\ref{alg:cls}) for~\eqref{eq:LS_optimization}, by using  $g(\cdot)$ as the oracle. We have the following worst-case performance guarantee for the objective:
\begin{equation}
    0 \leq g(\sopt) \leq 4(1+\alpha) g(\sls).
\end{equation}
\end{prop}
\begin{proof}
In~\eqref{eq:sensing_and_energy_objective}, mutual information is a submodular set function defined on measurements provided by selected trajectories~\cite{atanasov2015decentralized}.
Moreover, $C(S)$ is modular given its additive nature:
\begin{equation}
    C(S) = \sum_{\sigma_i \in S} r_i C_{i}( \sigma_i)\geq 0.
\end{equation}
Since mutual information is non-negative, \eqref{eq:LS_optimization} is a submodular non-monotone maximization problem with a partition matroid constraint. 
Setting $k=1$ and $\epsilon=\alpha$ in~\cite[Thm. 4]{lee2009non}, the proposition follows directly after rearranging terms.
\end{proof}

\begin{remark}
Having the constant $\lsoffset$ term in~\eqref{eq:LS_optimization} does not change the optimization in Problem~\ref{prob:MA_fuel_aware}, but ensures that the oracle used by \cls{} (Alg.~\ref{alg:cls}) is non-negative so that the ratio $(1+\frac{\alpha}{N^4})$ correctly reflects the sufficient improvement condition. 
\end{remark}

Besides the communication aspect that \cls{} requires access to all robot trajectories, running it naively can incur significant computation. In the worst case, \cls{} requires $\mathcal{O}(\frac{1}{\alpha}N^6\log(N))$ oracle calls\footnote{For 2 solution candidates, each requires $\mathcal{O}(\frac{1}{\alpha}N^4\log(N))$ local operations, and $N^2$ oracle calls to find each local operation in the worst case.}, where $N$ is the total number of trajectories~\cite{lee2009non}. 
Even on a central server, run-time may be greatly reduced by using our proposed method (see Sec.~\ref{sec:experiments}).

\subsection{Distributed Local Search (\texttt{DLS})} \label{sec:dls}
This section proposes a distributed implementation of local search
(see Algs.~\ref{alg:dls} and~\ref{alg:local_exchange} written for robot $i$).
Exploiting the structure of the partition matroid, \dls{} enables each robot to propose local operations based on its own trajectory set, while guaranteeing that the team solution never contains more than one trajectory for every robot.
All steps executed by \cls{} can be distributedly proposed, so \dls{} provides the same performance guarantee in Theorem~\ref{thm:LS_guarantee}.
By prioritizing search orders and starting with greedy solutions, we reduce computation and communication of \dls{}, respectively.

\subsubsection{Distributed Proposal}
Every proposal consists of two trajectories $(d,a)$, where $d$ is to be deleted from and $a$ is to be added to the solution set.
We also define a special symbol ``$\nop$'' that leads to no set operation, \ie, $S_k \cup \{\nop\} = S_k \backslash \{\nop\} = S_k$. 
Note that $(d,\nop)$, $(\nop, a)$ and $(d, a)$ are equivalent to the \textbf{Delete}, \textbf{Add} and \textbf{Swap} steps in \cls{}.

Every robot $i$ starts by sharing the size of its trajectory set $|\mathcal{M}_i|$ and its best trajectory $a^*_i \in \mathcal{M}_i$ in order to initialize $S_k$ and $N$ collaboratively (Alg.~\ref{alg:dls} Lines~\ref{dls:start_N_Sk}--\ref{dls:end_N_Sk}).
Repeatedly, every robot $i$ executes the subroutine \lexchange{} (Alg.~\ref{alg:local_exchange}) in parallel, in order to propose changes to $S_k$ (Alg.~\ref{alg:dls} Lines~\ref{dls:start_find_proposal}--\ref{dls:end_find_proposal}). Since any valid proposal shared by robots improves the objective, the first $(d, a)\neq (\nop, \nop)$ will be used by all robots to update $S_k$ in every round (Alg.~\ref{alg:dls} Lines~\ref{dls:start_receive_proposal}--\ref{dls:end_receive_proposal}).
We assume instantaneous communication, so robots always use a common proposal to update their copies of $S_k$. Otherwise, if delay leads to multiple valid proposals, a resolution scheme is required to ensure robots pick the same proposal.

In \lexchange{} (Alg.~\ref{alg:local_exchange}), an outer loop looks for potential deletion $d\in S_k$ (Alg.~\ref{alg:local_exchange} Lines~\ref{lexchange:start_delete}--\ref{lexchange:end_delete}). Otherwise, further adding $a\in\mathcal{M}_i$ is considered, as long as the partition matroid constraint is not violated (Alg.~\ref{alg:local_exchange} Lines~\ref{lexchange:start_check_partition_matroid}--\ref{lexchange:end_check_partition_matroid}). Next, we discuss how to efficiently search for trajectories to add.

\begin{algorithm}[t]
\caption{Distributed Local Search (\dls)}
\label{alg:dls}
\small %
\begin{algorithmic}[1]
\State \textbf{require} $\alpha>0$, trajectories $\mathcal{M}_i$, oracle $g$
\State Sort $\mathcal{M}_i$ in descending order based on $g(a|\emptyset)$ for all $a\in\mathcal{M}_i$\label{dls:pq_sort}
\State $S_1, S_2 \leftarrow \emptyset$
\For{$k=1, 2$}
    \State Broadcast $| \mathcal{M}_i |$ and $a_i^*\in\mathcal{M}_i$ that maximizes $ g(\{a_i^* \})$\label{dls:start_N_Sk}
    \State $S_k \leftarrow \{a^*\}$, where $a^*\in \{ a_i^* \}_{i=1}^n$ maximizes $ g(\{a^* \})$
    \State $N \leftarrow \sum_{i=1}^n | \mathcal{M}_i |$\label{dls:end_N_Sk}

    \Repeat\label{dls:start_find_proposal}
        \State Run \lexchange($S_k, \mathcal{M}_i,\alpha,N, g$) in background\label{dls:run_find_proposal}
        \If{Receive $(d, a) \neq (\nop, \nop)$}\label{dls:start_receive_proposal}
            \State Terminate \lexchange{} if it has not finished
            \State $S_k \leftarrow S_k \backslash \{ d \} \cup \{a\}$\label{dls:end_receive_proposal}
        \EndIf
    \Until{Receive $(d, a) = (\nop, \nop)$ from all robots}\label{dls:end_find_proposal}
    
    \State $\mathcal{M}_i \leftarrow \mathcal{M}_i \backslash S_k$
\EndFor %
\State \textbf{return} $ \argmax_{ S \in \{S_1, S_2\}} g( S)$
\end{algorithmic}
\end{algorithm}
\begin{algorithm}[t]
\caption{Find Proposal (\lexchange)}
\label{alg:local_exchange}
\small
\begin{algorithmic}[1]
\State 
        \textbf{require}
         $S_k$, %
         $\mathcal{M}_i$, %
         $\alpha>0$, $N$,
         $g$

\For{$d\in S_k$ or $d=\nop$}\ \algcomment{Delete $d$, or no deletion}\label{lexchange:start_delete}
    \State $S_k^- \leftarrow S_k \backslash \{d\}$ \label{lexchange:SK_minus}
    \State $\Delta \leftarrow (1+\frac{\alpha}{N^4}) g(S_k) - g(S_k^-)$\ \algcomment{$\Delta$: deficiency of $S_k^-$}\label{lexchange:deficiency_of_SK_minus}
    \If{$\Delta \leq 0$}
        \State \textbf{broadcast} $(d,\nop)$ \label{lexchange:end_delete}
    \EndIf
    \If{$\exists\ a\in S_k^-$ planned by robot $i$}\label{lexchange:start_check_partition_matroid}
        \State \textbf{continue} \algcomment{Cannot add due to partition matroid}\label{lexchange:end_check_partition_matroid}
    \EndIf
    \For{$a\in\mathcal{M}_i$ in sorted order} \algcomment{Add $a$}\label{lexchange:iter_sorted_order}
        \If{$g(a|\emptyset) < \Delta$}\label{lexchange:no_promising_trajs}
            \State \textbf{break} \algcomment{No $a\in\mathcal{M}_i$ will improve $S_k^-$ enough}\label{lexchange:break_if_no_promising_trajs}
        \EndIf
        \If{$g(a|S_k^-) \geq \Delta$}\label{lexchange:found_proposal} 
            \State \textbf{broadcast} $(d,a)$\label{lexchange:broadcast_proposal}
        \EndIf
    \EndFor
\EndFor
\State \textbf{broadcast} $(\nop, \nop)$

\end{algorithmic}
\end{algorithm}

\subsubsection{Lazy Search}
Instead of searching over trajectories in an arbitrary order, we can prioritize the ones that already perform well by themselves, based on $g( a | \emptyset)$ for all $a\in\mathcal{M}_i$ (Alg.~\ref{alg:dls} Line~\ref{dls:pq_sort}). In this fashion, we are more likely to find trajectories that provide sufficient improvement earlier (Alg.~\ref{alg:local_exchange} Lines~\ref{lexchange:found_proposal}--\ref{lexchange:broadcast_proposal}). 
Note that $g( a | \emptyset)$ is typically a byproduct of the trajectory generation process, so it can be saved and reused.

This ordering also allows us to prune unpromising trajectories.
Given the team solution after deletion $S_k^- \defeq S\backslash \{d\}$, the required marginal gain for later adding trajectory $a$ is
\begin{equation}\label{eq:suff_improv_a}
    g(a|S_k^-) \geq \Delta\defeq (1+\frac{\alpha}{N^4}) g(S_k) - g(S_k^-).
\end{equation}
We can prune any $a\in\mathcal{M}_i$, if $g( a | \emptyset)<\Delta$ based on the diminishing return property: because $\emptyset \subseteq S_k^- $, we know that $\Delta > g( a | \emptyset) \geq g( a | S_k^-)$, violating condition~\eqref{eq:suff_improv_a}.
Similarly, all subsequent trajectories $a'$ can be ignored, because their marginal gains $g(a'|\emptyset)\leq g(a|\emptyset)<\Delta$ due to ordering (Alg.~\ref{alg:local_exchange} Lines~\ref{lexchange:no_promising_trajs}--\ref{lexchange:break_if_no_promising_trajs}). 
Lastly, if an addition improves $S_k^-$ sufficiently, the proposal is broadcasted (Alg.~\ref{alg:local_exchange} Lines~\ref{lexchange:found_proposal}--\ref{lexchange:broadcast_proposal}). 

\subsubsection{Greedy Warm Start}
We observe empirically that a robot tends to swap its own trajectories consecutively for small growth in the objective, increasing communication unnecessarily. 
This can be mitigated by a simple technique: when finding local operations initially, we force robots to only propose additions to greedily maximize the objective, until doing so does not lead to enough improvement or violates the matroid constraint. Then robots resume Alg.~\ref{alg:local_exchange} and allow all local operations. By warm starting the team solution greedily, every robot aggregates numerous proposals with smaller increase in the objective into a greedy addition with larger increase, thus effectively reducing communication.

\section{Simulation Results}\label{sec:experiments}
We evaluate \dls{} in two target tracking scenarios based on objective values, computation, communication, and ability to handle heterogeneous robots. Its performance is compared against coordinate descent (\cd{}~\cite{atanasov2015decentralized}), a state-of-the-art algorithm for multi-robot target tracking that, however, assumes monotonicity of the objective. Planning for robots sequentially, \cd{} allows every robot to incorporate the plans of previous robots. We also allow \cd{} to not assign anything to a robot if it worsens the objective. Reduced value iteration~\cite{atanasov2014information} is used to generate trajectories for both algorithms.
Comparisons between \cls{} and \dls{} are omitted because the two algorithms empirically achieve the same average performance. We set $\alpha=1$ arbitrarily, because tuning it was not effective due to the large number of trajectories $N$.

Both \dls{} and \cd{} are implemented in C++ and evaluated in simulation on a laptop with an Intel Core i7 CPU.  For \dls{}, every robot owns separate threads, and executes Alg.~\ref{alg:local_exchange} over 4 extra threads to exploit its parallel structure. Similarly, \cd{} allows every robot to use 4 threads and additionally incorporates accelerated greedy~\cite{minoux1978accelerated} for extra speed-up.

\subsection{Characteristics of Robots}
Given initial state $x_{i,0} \in \mathcal{X}_i$ for robot $i\in\mathcal{R}_S$ who follows the control sequence $ u_{i,0},\dots,u_{i,T-1}=\sigma_i\in S$, the resultant states are $x_{i,1}, \dots, x_{i,T}$ based on dynamics~\eqref{eq:robot_dynamics}. The energy cost $C(S)$ may also be state-dependent. We define it as:
\begin{equation}\label{exp:energy_cost}
    C(S) \defeq \sum_{i\in\mathcal{R}_S} r_i \sum_{t=0}^{T-1} \left( \cctrl_i(u_{i,t}) + \cstate_i(x_{i,t}) \right),
\end{equation}
where the state-dependent cost $\cstate_i(\cdot)$ and control-dependent cost $\cctrl_i(\cdot)$ are defined based on robot types---in our case, robot $i$ is either an unmanned ground vehicle (UGV) or an unmanned aerial vehicle (UAV).  
Note that decomposition between state and control is not required for our framework to work.
The setup for robots are summarized in Table~\ref{tab:robot_setup}.
For simplicity, all robots follow differential-drive dynamics\footnote{We note that any dynamically feasible model can be used for the specific robot which is being planned for. We use the same kinematic model for the quadrotor and ground vehicle for implementation convenience, and because the quadrotors are restricted to a plane to avoid collisions.}
with sampling period $\tau=0.5$ and motion primitives consisting of linear and angular velocities $\{u=(\nu, \omega) \mid \nu \in \{0,8\}\text{ m/s},\ \omega\in\{0,\pm\frac{\pi}{2}\}\text{ rad/s}\}$. We consider muddy and windy regions that incur state-dependent costs for UGVs and UAVs, respectively.
The robots have range and bearing sensors, whose measurement noise covariances grow linearly with target distance. Within limited ranges and field of views (FOVs), the maximum noise standard deviations are $0.1\text{ m}$ and $5^\circ$ for range and bearing measurements, respectively. Outside the ranges or field of views, measurement noise becomes infinite. %
Please refer to~\cite{schlotfeldt2018anytime} for more details.

\begin{table}[b]
    \caption{Robot setup in two experimental scenarios.}\label{tab:robot_setup}
    \centering
        \resizebox{\columnwidth}{!}{%
            \bgroup
            \def\arraystretch{1.3}%
            \begin{tabular}{@{\extracolsep{2pt}} m{0.3cm} ccc cc c c @{}} %
            \toprule
            \multirow{2}{*}{} & \multicolumn{3}{c}{$\cctrl(u)$, $u$ given as} & \multicolumn{2}{c}{$\cstate(x)$, $x$ in} & FOV ($^\circ$) & \multicolumn{1}{c}{Range (m)}\tabularnewline
            \cline{2-4} \cline{5-6} \cline{7-7} \cline{8-8} 
             &$0,0$ & $0,\frac{\pm\pi}{2}$ & $8,\frac{\pm\pi}{2}$ & Mud & Wind & Exp.1\&2 & Exp.1\&2 \tabularnewline 
            \midrule
            UGV & 0 & 1 & 2 & 3 & / & 160 & 6\ \&\ 15\tabularnewline
            UAV & 2 & 2 & 4 & / & 3 & 360 & /\ \&\ 20\tabularnewline
            \bottomrule
            \end{tabular}
            \egroup%
        }
\end{table}

\subsection{Scenario 1: Multi-Robot Dynamic Target Tracking}
Here we show the computation and communication savings for \dls{}, and compare the performance of \dls{} and \cd{} (see Figs.~\ref{fig:dls_analysis} and \ref{fig:dls_obj_control_sensing_time}).
The scenario involves $2,\dots,10$ UGVs trying to estimate the positions and velocities of the same number of dynamic targets. The targets follow discretized double integrator models corrupted by Gaussian noise, with a top speed of $2$~m/s. Robots and targets are spawned in a square arena whose sides grow from $40\text{ m}$ to $60\text{ m}$, and $50$ random trials are run for each number of robots.

\begin{figure} [t!]
\vspace*{0.1in} %
\centering
\includegraphics[width=0.8\columnwidth, trim=10 0 45 13, clip]{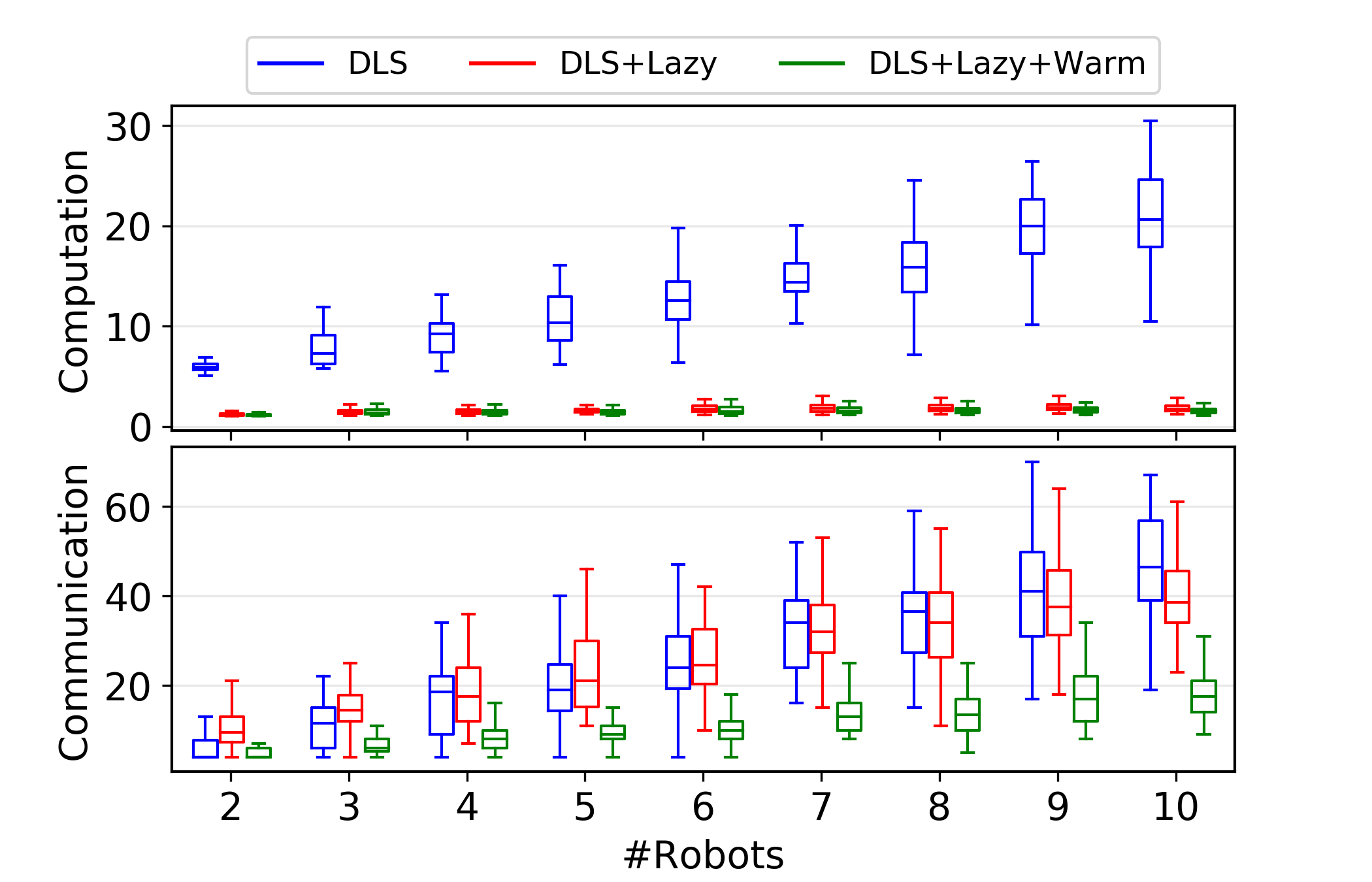}
\vspace*{-.1in}
\caption{Computation and communication savings afforded by lazy search (\textit{Lazy}) and greedy warm start (\textit{Warm}) for \dls{}.
Computation is measured by total oracle calls divided by the number of trajectories $N$, where $N$ reaches around $12500$ for $10$ robots. Communication is measured by the number of proposal exchanges. 
Combining lazy search and greedy warm start (green) leads to 80--92\% computation reduction, and up to 60\% communication reduction compared to the naive implementation (blue) on average.
}\label{fig:dls_analysis}
\end{figure}
\begin{figure} [t!]
\centering
\includegraphics[width=0.99\columnwidth, trim=20 15 20 0, clip]{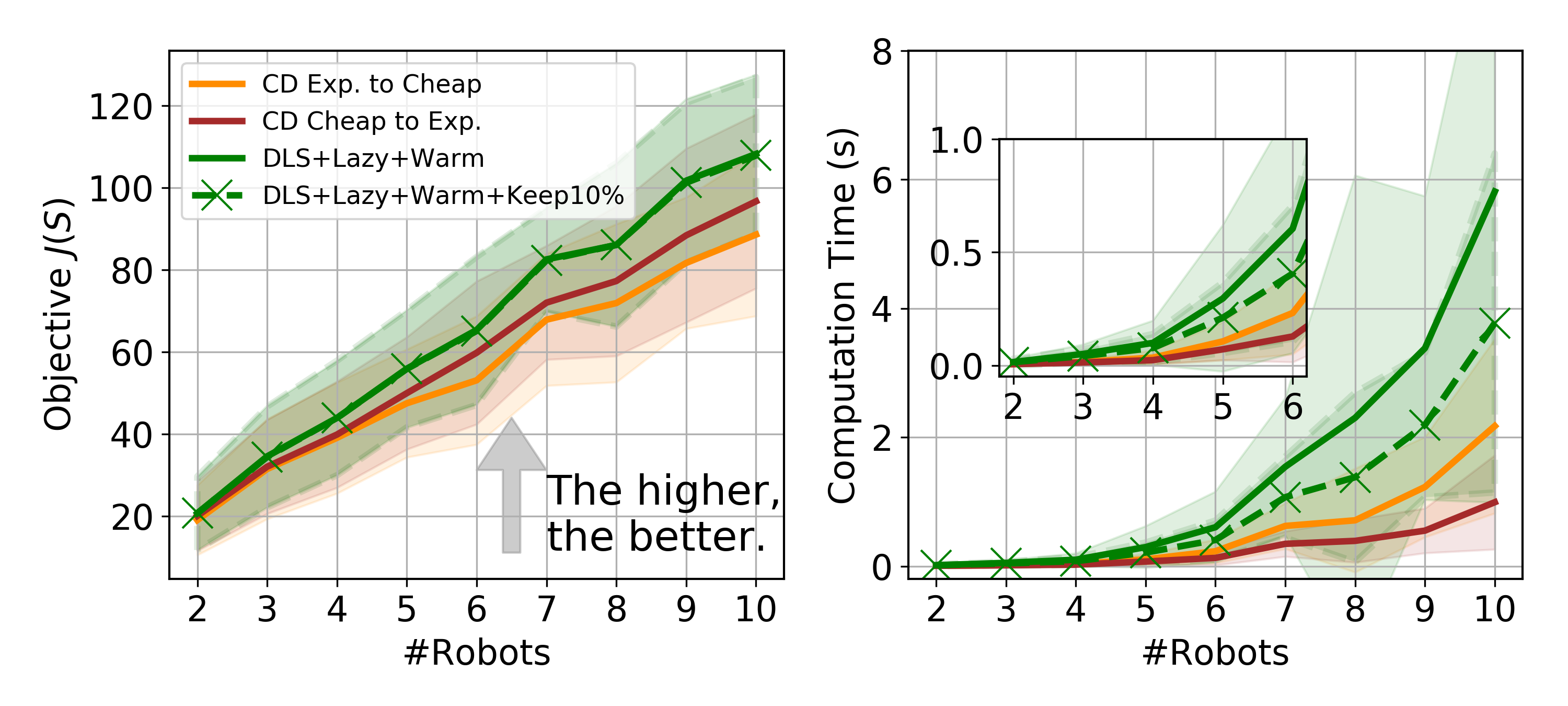}
\vspace*{-.25in}
\caption{
Objective values and computation time (s) for variants of \dls{} and \cd{}, where the lines and shaded areas show the mean and standard deviation, respectively. The time excludes the trajectory generation time ($<2$\text{ s}), which is the same for every algorithm.
\dls{} (solid green) consistently outperforms \cd{} in optimizing the objective, where it is better for \cd{} to plan from cheaper to more expensive robots (brown), rather than the reverse order (orange). The performance gap between \dls{} and \cd{} widens as more costly robots increase non-monotonicity of the problem.
However, \dls{} requires longer run-time, which in practice can be alleviated by using a portion of all trajectories. This invalidates the worst-case guarantee, but \dls{} solution based on the best $10\%$ of each robot's trajectories (green crosses) still outperforms \cd{}.}\label{fig:dls_obj_control_sensing_time}
\end{figure}

Non-monotonicity in the problem is accentuated by an increasing penalty for control effort of additional robots,
by setting  $r_i=i$ for each robot $i$ as defined in~\eqref{exp:energy_cost} (\ie, the $10$-th added robot is $10$ times more expensive to move than the first). Note that state-dependent cost is set to $0$ only for this experiment.
Trajectory generation has parameters $\epsilon=1$ and $\delta=2$ for horizon $T=10$. 
As the planning order is arbitrary for \cd{}, we investigate two planning orders: first from cheaper to more expensive robots, and then the reverse.
Intuitively and shown in Fig.~\ref{fig:dls_obj_control_sensing_time}, the former should perform better, because the same amount of information can be gathered while spending less energy. While other orderings are possible (\eg,\cite{jorgensen2018team,dames2017detecting}), we only use two to show \cd{}'s susceptibility to poor planning order. 
For a fair comparison between \dls{} and \cd{}, we use a fixed set of trajectories generated offline, but ideally trajectories should be replanned online for adaptive dynamic target tracking.

Proposed methods for improving naive distributed execution of local search, namely lazy search (\textit{Lazy}) and greedy warm start (\textit{Warm}), are shown to reduce computation by 80--92\% and communication by up to 60\% on average, as shown in Fig.~\ref{fig:dls_analysis}. 
As expected, when there are few robots with similar control penalties, the objective is still close to being monotone, and \dls{} and \cd{} perform similarly as seen in Fig.~\ref{fig:dls_obj_control_sensing_time}. However, as more costly robots are added, their contributions in information gain are offset by high control penalty, so the problem becomes more non-monotone. Therefore, the performance gap between \dls{} and \cd{} widens, because \cd{} requires monotonicity to maintain its performance guarantee, but \dls{} does not.
From Fig.~\ref{fig:dls_obj_control_sensing_time}, we can see that planning order is critical for \cd{} to perform well, but a good ordering is often unknown a priori. 
Compared to \cd{} which requires only $n-1$ communication rounds for $n$ robots, \dls{} requires more for its performance. 
For practical concerns to save more time, \dls{} with down-sampled trajectories (\eg, keeping the best $10\%$ of each robot's trajectories) still produces better solution than \cd{}, but the guarantee of \dls{} no longer holds.

\begin{figure} [tp!]
\vspace*{.08in} %
\centering
\includegraphics[width=0.7\columnwidth, trim=0 0 20 25, clip]{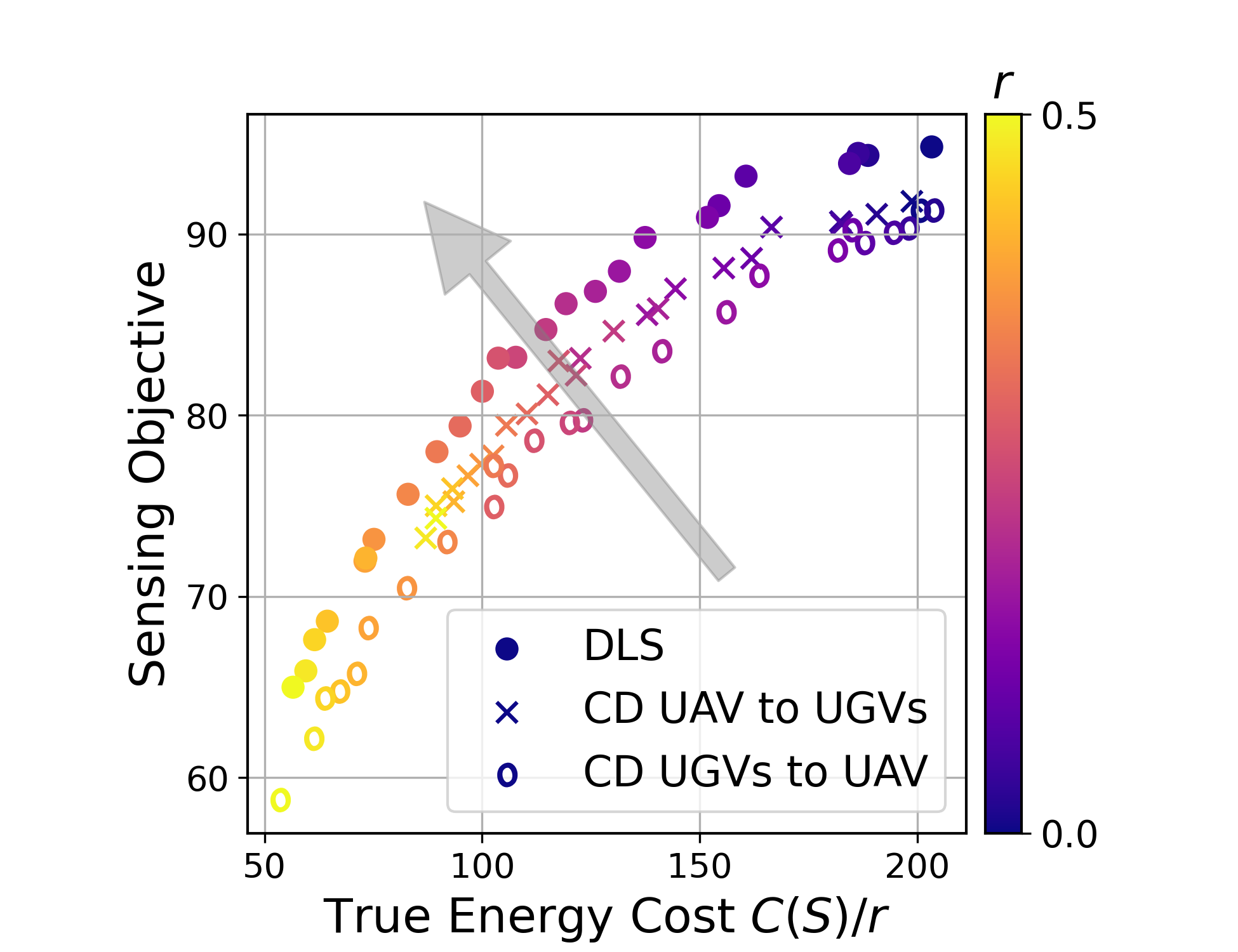}
\caption{Trade-off between sensing performance (mutual information~\eqref{eq:sensing_and_energy_objective}) and the true energy expenditure $C(S)/r$ in heterogeneous robot experiments produced by \dls{} and \cd{}, where it is better to be in the upper left. Each point is an average obtained over $50$ trials for a fixed $r$, where we set $r_i=r$ for each robot $i$ to penalize the team energy expenditure per~\eqref{exp:energy_cost}.}
\label{fig:hetero_tradeoff}
\end{figure}

\begin{figure} [tp!]
\vspace*{.08in} %
\centering
\includegraphics[width=0.95\linewidth, trim=85 55 65 30, clip]{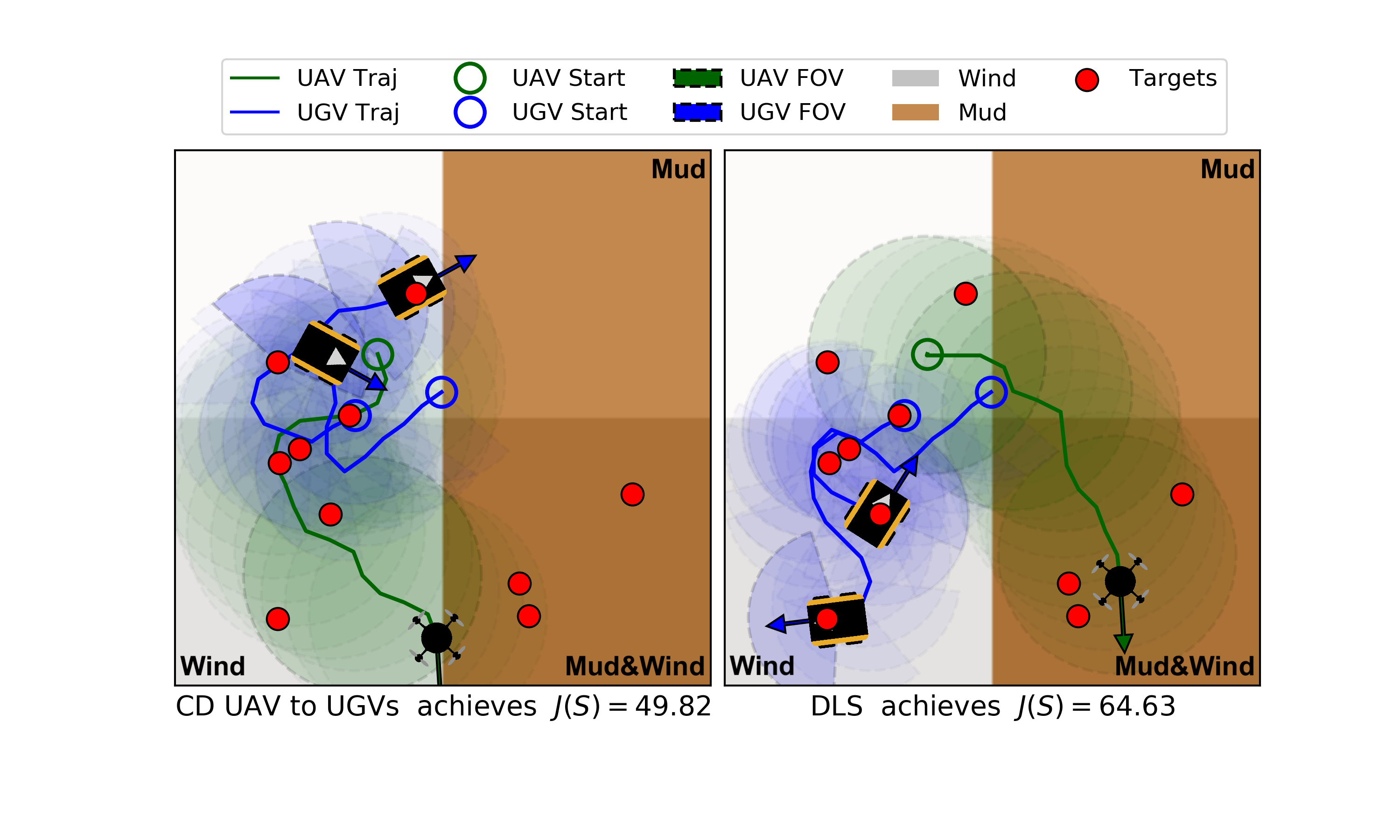}
\vspace*{-.1in}
\caption{Example solutions from \cd{} (left) and \dls{} (right) for 2~UGVs and 1~UAV with $r=0.2$ that penalizes energy cost $C(S)$ in~\eqref{exp:energy_cost}. 
The arena is both windy and muddy, which is costly for the UAV and UGVs, respectively.
(Left) \cd{} performs poorly due to its fixed planning order: the UAV plans first to hover near the targets on the left, rather than venturing over the mud. Thus, the UGVs are under-utilized because they are unwilling to go into the mud to observe the targets on the bottom right. For similar reasons, \cd{} with reversed order under-utilizes the UAV, which is not visualized due to limited space. (Right) In contrast, \dls{} deploys the UAV over the muddy regions, leading to a better value of $J(S)$ in \eqref{eq:sensing_and_energy_objective}.}
\label{fig:hetero_compare}
\end{figure}

\subsection{Scenario 2: Heterogeneous Sensing and Control} \label{sec:exp_hetero}

Now consider a heterogeneous team with 2 UGVs and 1 UAV with different sensing and control profiles (Table~\ref{tab:robot_setup}) tracking $10$ static targets in a $100\text{ m}\times100\text{ m}$ arena over a longer horizon $T=20$ (see Fig.~\ref{fig:hetero_compare}). The UAV has better sensing range and field of view compared to UGVs, but consumes more energy.
The arena has overlapping muddy and windy regions, so robots must collaboratively decide which should venture into the costly regions.
To explore the trade-off between sensing and energy objectives as a team, we set 
$r_i=r, ~\forall i$ and then, as we vary $r$ from $0$ to $0.5$, we run $50$ trials for each value. Robots are spawned in the non-muddy, non-windy region, but targets may appear anywhere. 
We set $\delta=4$ to handle the longer horizon, and evaluate two \cd{} planning orders: from UAV to UGVs, and the reverse.

As shown in Fig.~\ref{fig:hetero_tradeoff}, \dls{} consistently achieves better sensing and energy trade-off than \cd{} on average. To gain intuitions on why \cd{} under-performs, a particular trial given $r=0.2$ is shown in Fig.~\ref{fig:hetero_compare}. Due to the non-monotone objective, the robot who plans first to maximize its own objective can hinder robots who plan later, thus negatively affecting team performance.

\balance
\section{Conclusion}\label{sec:conclusion}
This work considered a multi-robot information gathering problem with non-monotone objective that captures the trade-off between sensing benefits and energy expenditure. We proposed a distributed algorithm based on local search and reduced its computation and communication requirements by using lazy and greedy methods. The proposed algorithm was evaluated in two target tracking scenarios and outperformed the state-of-the-art coordinate descent method. Future work will focus on scaling the algorithm to larger robot teams by exploiting spatial separation, formalizing heterogeneity, and carrying out hardware experiments.

\section*{Acknowledgment}
This research was supported in part by Boeing Research \& Technology and ARL DCIST CRA W911NF-17-2-0181.%
\clearpage

\appendices

\bibliographystyle{IEEEtran}
\bibliography{bibs}
\end{document}